\documentclass[11pt]{article}
\usepackage{graphicx}

\newcommand{\mfrac}[2]{%
  \mathord{\scalebox{0.90}{$\tfrac{#1}{#2}$}}%
}
\usepackage{geometry}
\usepackage{amsmath,amssymb,amsthm,mathtools}
\usepackage{microtype}
\usepackage{xcolor}
\usepackage{xparse}
\usepackage{fullpage}
\definecolor{suggestioncolor}{RGB}{0,92,175}
\NewDocumentCommand{\suggestion}{m m O{}}{%
    \begingroup
    \color{suggestioncolor}#2%
    \endgroup
}
\usepackage[
    backend=biber,
    style=numeric,
    sorting=nyt,
    backref=true,
    doi=false,
    url=false,
    eprint=false,
    maxbibnames=99
]{biblatex}
\usepackage{aliascnt}
\usepackage[
    colorlinks=true,
    linkcolor=blue,
    citecolor=blue,
    urlcolor=blue,
    pageanchor
]{hyperref}
\usepackage[nameinlink]{cleveref}

\usepackage{algorithm}
\usepackage{algpseudocode}

\DeclareMathOperator{\Ln}{Ln}
\DeclareMathOperator{\loss}{\cL}
\DeclareMathOperator{\ind}{\mathbf I}
\DeclareMathOperator{\sign}{\mathit{sign}}

\DeclareMathOperator{\p}{\mathbb P}
\DeclareMathOperator{\conv}{\operatorname{conv}}
\newtheorem{theorem}{Theorem}

\newaliascnt{lemma}{theorem}
\newtheorem{lemma}[lemma]{Lemma}
\aliascntresetthe{lemma}

\crefname{lemma}{lemma}{lemmas}
\Crefname{lemma}{Lemma}{Lemmas}

\makeatletter
\newcommand{\alphacmd@factory}[1]{}
\newcounter{alphacmdcounter}
\newcommand{\GenerateAlphabetCmds}[2]{%
    \renewcommand{\alphacmd@factory}[1]{%
        \expandafter\providecommand\csname #1##1\endcsname{{#2{##1}}}%
    }
    \setcounter{alphacmdcounter}{0}
    \loop
        \stepcounter{alphacmdcounter}
        \edef\alphacmd@ID{\@Alph\c@alphacmdcounter}
        \expandafter\alphacmd@factory\alphacmd@ID
    \ifnum\thealphacmdcounter<26
    \repeat
}
\newcommand{\GenerateAlphabetCmdsLower}[2]{%
    \renewcommand{\alphacmd@factory}[1]{%
        \expandafter\providecommand\csname #1##1\endcsname{{#2{##1}}}%
    }
    \setcounter{alphacmdcounter}{0}
    \loop
        \stepcounter{alphacmdcounter}
        \edef\alphacmd@ID{\@alph\c@alphacmdcounter}
        \expandafter\alphacmd@factory\alphacmd@ID
    \ifnum\thealphacmdcounter<26
    \repeat
}
\makeatother

\GenerateAlphabetCmds{c}{\mathcal}
\GenerateAlphabetCmdsLower{c}{\mathcal}

\GenerateAlphabetCmds{r}{\mathbf}
\GenerateAlphabetCmdsLower{r}{\mathbf}

\title{\Large\bfseries Tight Generalization Bound for AdaBoost}
\vspace{-1.5em}
\author{
  \normalsize Mikael Møller Høgsgaard\\[-0.1em]
  \normalsize Department of Statistics, University of Oxford
}
\date{}
\begin{document}

\maketitle

\begin{abstract}
\linespread{1.0}\selectfont
\noindent
In this paper we show that the generalization error of AdaBoost is $\Theta\big(\mfrac{d\ln(n\gamma^{2}/d)}{n\gamma^2}+\mfrac{\ln(1/\delta)}{n}\big)$, where $\gamma$ is the advantage guaranteed by the weak learner, $d$ is the VC-dimension of the class containing the weak hypotheses, $n$ is the sample size, and $\delta$ is the confidence parameter. The contribution of this paper is the upper bound; the matching lower bound follows from prior work. The upper bound proof follows by combining the known fact that AdaBoost outputs a voting classifier whose voting function has zero empirical $\gamma/2$-margin loss with what is, to the best of our knowledge, a new margin-based generalization bound for voting classifiers.
\end{abstract}

\section{Introduction}

 An empirical weak learner is a learning algorithm that can produce a hypothesis with accuracy slightly better than random guessing on any reweighting of a given data set. Formally, an algorithm $ \cW $  is an empirical $\gamma$-weak learner for a labeled sample
$S=((x_1,y_1),\ldots,(x_n,y_n))$  and a hypothesis class $\cH\subseteq\{-1,+1\}^{\cX}$ if, when given a distribution $D$ over the examples in $S$, it outputs a hypothesis $\cW(D)\in\cH$ such that
\[
\loss_{D}(\cW(D)):=\sum_{i=1}^{n}D(i)\ind\{\cW(D)(x_i)\ne y_i\}
\leq \frac12-\gamma.
\]
The notion of weak learning was introduced by Kearns and Valiant~\cite{KearnsV94}, who asked the interresting question whether a weak learner could be turned (boosted) into a strong learner. A strong learner~\cite{VapnikC71,Valiant84,BlumerEHW89} is a learning algorithm that can produce a hypothesis with accuracy arbitrarily close to 1 given enough training data. Formally, a strong learner for a distribution $P$ over the instance space $\cX$ is an algorithm that, given a sample $ S=((x_1,y_{1}),\ldots,(x_n,y_{n})) $ of size $n$, where the $x_i$ are drawn i.i.d. from $P$, outputs a predictor $ \cA(S):\cX\to\{-1,+1\} $ such that, with probability at least $1-\delta$ over $S$,
\[
\loss_{P}(\cA(S)) \le r(n,\delta),
\]
where $r(n,\delta)\to0$ as $n\to\infty$ for every fixed $\delta>0$.
The question of whether a weak learner can be turned into a strong learner was answered affirmatively by Schapire~\cite{Schapire89}, by giving an algorithm for doing so. Later, Freund and Schapire~\cite{FreundS95} introduced the famous AdaBoost, for which they received the Gödel Prize in 2003. AdaBoost
\footnote{Throughout this paper, AdaBoost refers to the version stated in
\Cref{alg:adaboost}, following Freund and Schapire
\cite[5]{SchapireF12}. As discussed in \cite[48]{SchapireF12}, an empirical weak learner suffices for AdaBoost to succeed, so we adopt this setting.} produces a strong learner from a weak learner by cleverly combining the weak hypotheses produced over $T$ rounds into a function $\sum_{t=1}^{T}\alpha_t h_t$, where $\alpha_t$ is the weight of the $t$-th weak hypothesis. Its final prediction is obtained by taking the sign of the function; the resulting predictor is called a \emph{voting classifier}. AdaBoost selects these hypotheses by sequentially calling the weak learner on distributions that place more weight on examples misclassified by the current voting classifier.

The general process of turning a weak learner into a strong learner is called boosting. AdaBoost is one of the best-known boosting algorithms; it pioneered the field and demonstrated the practical effectiveness of boosting. Since the introduction of AdaBoost, boosting methods have continued to evolve. Where modern boosting algorithms such as XGBoost~\cite{ChenG16}, LightGBM~\cite{KeMFYLM23}, and CatBoost~\cite{ProkhorenkovaGVDG18} are used for tabular and structured data, when for instance single-prediction latency is important (for a recent overview, see \cite{ericksonP25}).
Although AdaBoost is one of the most extensively studied boosting algorithms, its generalization error is not yet fully understood. In this paper we establish that it's generalization error is
\begin{align*}
  \Theta\left(
\frac{d\ln(n\gamma^2/d)}{n\gamma^2}
+\frac{\ln(1/\delta)}{n}
\right).
\end{align*}
We establish this result by proving that the generalization error of AdaBoost is at most the above, which combined with the lower bound of \cite{HogsgaardLR23} matching the above yields the conclusion.
More formally, we prove the following upper bound (see \Cref{thm:adaboost-generalization} for a formal statement):
\begin{quote} Let $P$ be a distribution over $\cX\times \{-1,+1\}$, and suppose that there exists an algorithm $\cW$ that is an empirical $\gamma$-weak learner for every labeled sample $S=((x_1,y_1),\ldots,(x_n,y_n))$ of size $n$ whose instances $(x_i,y_i)$ are drawn i.i.d. from $P$. If the hypotheses output by $\cW$ belong to a class $\cH$ of VC-dimension $d$, then, with probability at least $1-\delta$ over $S$, AdaBoost, run for $T\geq \ln(n)/\gamma^2$ rounds, outputs a voting classifier with generalization error at most $$O\left(
\frac{d\ln(n\gamma^2/d)}{n\gamma^2}
+\frac{\ln(1/\delta)}{n}
\right).$$
\end{quote}
The upper-bound proof follows by combining the known fact that AdaBoost outputs a voting classifier whose voting function has zero empirical $\gamma/2$-margin loss with what is, to the best of our knowledge, a new margin-based generalization bound for voting classifiers; we derive this bound using a small modification of standard VC-theoretic techniques.

As mentioned the tightness of this upper bound is witnessed by the work \cite{HogsgaardLR23}. More specifically, \cite{HogsgaardLR23} showed that for every VC-dimension $d$, margin $\gamma$, and sample size $n$ satisfying $ \gamma\leq c_{1} $,  $ d\geq c_{2} \ln{(1/\gamma)} $ and $ c_{3}d\gamma^{-2}\leq n \leq  \exp{( \exp{(c_{4}d )}  )}  $, where $ c_{1},c_{2},c_{3},c_{4} $ are positive universal constants, there exist a distribution $P$ over the labeled example space $\cX\times\{-1,1\}$ and an empirical $\gamma$-weak learner $\cW:\Delta((\cX\times\{-1,1\})^n)\to\cH$, where $\cH$ has VC-dimension $d$, such that the generalization error of AdaBoost is at least $$ \Omega\left(
\frac{d\ln(n\gamma^2/d)}{n\gamma^2}
+\frac{\ln(1/\delta)}{n}
\right) ,$$
which gives us the desired tight result up to constants, and in the specified range of the lower bound.

\begin{algorithm}[H]
    \caption{AdaBoost \cite[6]{SchapireF12}}
    \label{alg:adaboost}
    \begin{algorithmic}[1]
        \Statex Given $ (x_{1},y_{1}),\ldots,(x_{n},y_{n}) $, where $ x_{i}\in \cX $ and $ y_i\in\{-1,1\} $
        \Statex Initialize $ D_{1}(i)=\frac{1}{n} $ for $ i=1,\ldots,n $
        \For{$ t=1,\ldots,T $}
            \State Call the weak learner $ \cW $ with the $ D_{t} $-weighted sample to get hypothesis $ h_{t}:\cX\rightarrow \{-1,1\} $
            \State Compute the error of $ h_{t} $ with respect to $ D_{t} $:
            \[
            \epsilon_{t}=\sum_{i=1}^{n}D_{t}(i)\ind\{h_{t}(x_{i})\ne y_{i}\}
            \]
            \State Compute the weight of $ h_{t} $:
            \[
            \alpha_{t}=\frac{1}{2}\ln{\left(\frac{1-\epsilon_{t}}{\epsilon_{t}}\right)}
            \]
            \State Update the distribution $ D_{t} $:
            \[
            D_{t+1}(i)=\frac{D_{t}(i)\exp(-\alpha_{t}y_{i}h_{t}(x_{i}))}{Z_{t}}
            \]
            where $Z_t$ is chosen so that $\sum_{i=1}^{n}D_{t+1}(i)=1$.
        \EndFor
        \State Output the final voting classifier:
        \[
        H(x)=\sign\left(\sum_{t=1}^{T}\alpha_{t}h_{t}(x)\right)
        \]
    \end{algorithmic}
\end{algorithm}
\vspace{-1em}
\section{Related Work}
The previous best upper bound on the generalization error of AdaBoost is due to \cite{HogsgaardG25} and is suboptimal by up to a factor of $\ln{(\ln{(n\gamma^{2}/d)})}^{2}$. This bound follows from a margin-based generalization bound for voting classifiers given in \cite{HogsgaardLR23}. In general, margin-based generalization bounds for voting classifiers have been used to explain the success and generalization properties of boosting algorithms that produce such classifiers. These bounds have been studied in many works, beginning with \cite{BartlettFLS98} and continuing with \cite{Breiman99,KoltchinskiiP02,SrebroST10,GaoZH13,GronlundKLMN19,GronlundKL20,HogsgaardLR23,LarsenS25}. For the special case in which the voters come from a finite hypothesis set, optimal margin-based generalization bounds were given in~\cite{LarsenS25}. These results do not yield a tight generalization-error bound for AdaBoost in full generality because the weak hypotheses need not come from a finite class. When the empirical $\gamma$-margin loss is zero, our new margin-based bound extends this result to hypothesis classes of finite VC-dimension. A separate line of work studies sample-compression bounds for AdaBoost, but these bounds are not known to be tight~\cite{cunhaLR25}. For general binary-classification boosting, the optimal generalization-error rate is known to be $ \Theta(d/(n\gamma^{2})+\ln(1/\delta)/n) $ \cite{HogsgaardLR23}, the classifier used here is not a voting classifier over $\cH$.
Boosting has also been studied in many other settings, including agnostic binary classification~\cite{Ben-DavidLM01,KalaiK09,Feldman10,BrukhimCHM20,GhaiS24,GhaiS25,DaCunhaHPS25,DaCunhaHP26}; multiclass boosting~\cite{SchapireS98,MukherjeeS13,BrukhimDM23}; regression~\cite{Drucker97,DuffyH02,Friedman01,BühlmannB03}; parallel boosting~\cite{LongS13,karbasiL24,XinHJ24,daCunhaHL24}; smooth boosting~\cite{Servedio01,Gavinsky02,Barak09,BunCS20}; replicable boosting~\cite{ImpagliazzoLPS22,larsenMS26}; and generalized boosting~\cite{bressan25a}, among many others. For an overview of boosting theory, see \cite{SchapireF12}.

\section{Setup and Notation}
In this section we introduce the notation and definitions used throughout the paper. We denote by $ \cX $ the instance space, and by $ \cH\subseteq\{-1,+1\}^{\cX} $ a class of voters.
For a hypothesis class $\cH$, we denote by $ \conv(\cH) $ the convex hull of $ \cH $, i.e.,  $ \conv(\cH)=\{\sum_{h\in\cH}\alpha_h h: \alpha_h\ge 0, \sum_{h\in\cH}\alpha_h=1\} $. We call an element $g\in\conv(\cH)$ a \emph{voting function}, and the induced voting classifier is $ G(x)=\sign(g(x)) $.
Given a distribution $P$ on $\cX\times\{-1,+1\}$ and an i.i.d.
sample $S=((x_i,y_i))_{i=1}^n\sim P^n$, the true risk and empirical risk of a classifier $ G: \cX\rightarrow \{-1,+1\} $  are
\[
\loss_P(G)=\p(G(X)\ne Y),\qquad
\loss_S(G)=\frac1n\sum_{i=1}^n
\ind\{G(x_i)\ne y_i\}.
\]
For a margin
parameter $\gamma\in[0,1]$, define the population and empirical margin losses
of a function $g: \cX\rightarrow \mathbb R$ by
\[
\loss_P^{\gamma} (g)=\p(Yg(X)\le \gamma),\qquad
 \loss_S^\gamma  (g)=\frac1n\sum_{i=1}^n
\ind\{y_i g(x_i)\le \gamma\}.
\]
We say that an algorithm $\cW$ is an empirical $\gamma$-weak learner with respect to a hypothesis class $\cH\subseteq\{-1,+1\}^{\cX}$ and a sample $ S $  if, for every distribution $D$ over the samples examples\footnote{The distribution $D$ assigns a weight to each example; the weak learner receives the weighted sample $((D(1),x_1,y_1),(D(2),x_2,y_2),\ldots,(D(n),x_n,y_n))$.}, it outputs a hypothesis $\cW(D)\in\cH$ such that
\[
\loss_{D}(\cW(D)):=\sum_{i=1}^{n}D(i)\ind\{\cW(D)(x_i)\ne y_i\}
\le \frac12-\gamma.
\]
We will use $  \Delta((\cX\times\{-1,1\})^{*}) $ to denote the set of all distributions over finite sequences of labeled examples from $ \cX\times\{-1,1\} $. Accordingly, $\cW$ is a mapping $ \cW: \Delta((\cX\times\{-1,1\})^{*})\rightarrow \cH $.
Throughout, $\Ln(t)=\max\{1,\ln t\}$ denotes the truncated logarithm. Furthermore, we let $ \sign(x)=-\ind\{x<0\}+ \ind\{x\ge 0\} $; that is, the sign of $0$ is set to $1$.

\section{Upper Bound on the Generalization Error of AdaBoost}
In this section, we prove the new generalization bound for AdaBoost. As mentioned in the introduction, the upper bound on the generalization error of AdaBoost follows from a new margin-based generalization bound for voting classifiers, which we state here and prove in the following section.

\begin{theorem}[Voting function margin bound]\label{thm:voting-margin-bound}
Let $\cH\subseteq\{-1,+1\}^{\cX}$ be a class of voters with VC-dimension $d$. There are
universal constants $c,C>0$ such that, for every $\gamma\in(0,1]$,
$\delta\in(0,1)$, distribution $P$ on $\cX\times\{-1,+1\}$, and
$n\ge 1$, the following holds with probability at least $1-\delta$ over
$S\sim P^n$: every voting function
$g\in\conv (\cH)$ satisfying
$ \loss_S^\gamma  (g)=0$ satisfies
\[
\loss_{P}^{0}(g)
\le
C\left(
\frac{d\Ln(c\gamma^{2}n/d)}{\gamma^{2}n}
+\frac{\ln(e/\delta)}{n}
\right),
\qquad
d_\gamma=c\frac{d}{\gamma^2}.
\]
\end{theorem}

Furthermore, we need the following known result about AdaBoost, which is a direct consequence of the analysis of AdaBoost given in \cite[111--113]{SchapireF12}.
\begin{lemma}[AdaBoost margin bound]\label{lem:adaboost-margin}
Let $n\geq 2$ and $0<\gamma<1/2$. Suppose that the errors in \Cref{alg:adaboost} satisfy $\epsilon_t\leq 1/2-\gamma$ for every $t=1,\ldots,T$. Define $A=\sum_{t=1}^{T}\alpha_t$ and the voting function $g(x)=\sum_{t=1}^{T}\frac{\alpha_t}{A}h_t(x)$. Then
\[
\loss_{S}^{\gamma/2}(g)\leq \left( \sqrt{(1-2\gamma)^{1-\gamma/2}(1+2\gamma)^{1+\gamma/2}}\right)^{T}.
\]
Moreover, if $ T\geq \frac{\ln{n}}{\gamma^{2}} $, then $ \loss_{S}^{\gamma/2}(g)=0 $.
\end{lemma}
The final conclusion $\loss_S^{\gamma/2}(g)=0$ for $T\geq\ln(n)/\gamma^2$ is not stated explicitly in \cite[111--113]{SchapireF12}; we verify it in \Cref{sec:appendix}. Combining \Cref{lem:adaboost-margin} with \Cref{thm:voting-margin-bound} now yields the paper's main result: an upper bound on AdaBoost's generalization error that matches the lower bound of \cite{HogsgaardLR23} up to universal constants.

\begin{theorem}[Generalization error of AdaBoost]\label{thm:adaboost-generalization}
Let $\cH$ be a hypothesis class with VC-dimension $d\geq 1$, and let
$\gamma\in(0,1/2)$, $\delta\in(0,1)$, $n\geq 2$, and $P$ be a distribution
on $\cX\times\{-1,+1\}$. Suppose that
$\cW\colon\Delta((\cX\times\{-1,1\})^{*})\to\cH$ is an empirical
$\gamma$-weak learner for every sample $S\sim P^n$. Then, with probability
at least $1-\delta$ over $S\sim P^n$, the voting classifier $H$ returned by
AdaBoost after $T\geq \ln(n)/\gamma^2$ rounds satisfies
\[
\loss_P(H)
=O\left(
\frac{d\Ln(n\gamma^2/d)}{n\gamma^2}
+\frac{\ln(1/\delta)}{n}
\right).
\]
\end{theorem}

\begin{proof}
Because $\cW$ is an empirical $\gamma$-weak learner, AdaBoost has $\epsilon_t\leq 1/2-\gamma$ in every round. Hence, by \Cref{lem:adaboost-margin}, after $T\geq\ln(n)/\gamma^2$ rounds, the normalized voting function $g(x)=\sum_{t=1}^{T}\frac{\alpha_t}{A}h_t(x)$, where $A=\sum_{t=1}^{T}\alpha_t$, belongs to $\conv(\cH)$ and satisfies $\loss_S^{\gamma/2}(g)=0$. Applying \Cref{thm:voting-margin-bound} with margin $\gamma/2$ then gives, with probability at least $1-\delta$,
\[
\loss_{P}^{0}(g)\leq O\left(\frac{d\Ln{(n\gamma^{2}/d)}}{n\gamma^{2}}+\frac{\ln{(1/\delta)}}{n}\right).
\]
Because $A>0$, normalization does not change the sign of the vote, so $H=\sign(g)$. Therefore
\begin{align*}
\loss_{P}(H)=\p(\sign(g(X))\ne Y)\leq \p(Yg(X) \leq 0)= \loss_{P}^{0}(g),
\end{align*}
where the inequality follows from the fact that $\sign(g(X))\ne Y$ implies $Yg(X)\leq 0$, while the reverse might not be true when $ g(X)=0 $.
This completes the proof.

\end{proof}

\section{Margin-Based Generalization Bound for Voting Classifiers}
In this section we prove the margin-based generalization bound for voting classifiers stated in \Cref{thm:voting-margin-bound}.
The new margin-based bound follows the classical VC-dimension generalization-bound proof technique of \cite{VapnikC71,Valiant84,BlumerEHW89} by introducing a ghost sample. The main difference is that we must bound the VC-dimension of a class of sets defined by the margin losses of voting functions. To do so, we use a slight variation of the standard ghost-sample approach, in which the ghost sample has the same size as the training sample. More specifically, our ghost sample is smaller than in the standard approach but \emph{importantly} has the property that every point in it has nonpositive margin. The nonpositive margin on the ghost sample allows us to apply a Rademacher-complexity argument, to obtain a tighter complexity bound than in previous approaches, and ultimately derive the desired bound. The smaller ghost sample is the main difference from previous methods and the main source of the improved bound.
\begin{proof}
Let $c>0$ be a universal constant, to be fixed by the VC-dimension argument below, and set
\[
d_\gamma=c\frac{d}{\gamma^2},
\qquad
\varepsilon
=10\left(
\frac{d_\gamma\Ln(6n/d_\gamma)}{n}
+\frac{\ln(e/\delta)}{n}
\right).
\]
We will show the claim in \Cref{thm:voting-margin-bound} with $\loss_{P}^{0}(g)\le C\varepsilon$.
If $\varepsilon\ge 1$, the theorem follows by taking $C\geq 10$, since
$\loss_P^0(g)\le 1$. Hence, assume $\varepsilon<1$. It suffices to show that
\[
\mathbb P_{S\sim P^n}\left(
\exists g\in\conv (\cH):
\loss_S^\gamma  (g)=0
\text{ and }
\loss_P^0(g)\ge \varepsilon
\right)\le \delta .
\]
We first relate the target bad event to an event witnessed on a smaller ghost sample. Let $S_1\sim P^n$ and $S_2\sim P^m$ be independent samples, where
$m=\lfloor \varepsilon n/2\rfloor$, and set $N=n+m$. Define the events
\[
\begin{aligned}
E_1
&=
\left\{
\exists g\in\conv (\cH):
\loss_{S_1}^\gamma (g)=0,\,
\loss_P^0(g)\ge \varepsilon
\right\},\\
E_2
&=
\left\{
\exists g\in\conv (\cH):
\loss_{S_1}^\gamma (g)=0,\,
\loss_{S_2}^0 (g)=1
\right\}.
\end{aligned}
\]
For every realization of $S_1$ for which $E_1$ occurs, fix a witnessing function $g$. Since
$\loss_P^0(g)\ge\varepsilon$, the probability that all $m$ points of $S_2$
have nonpositive margin is at least $\varepsilon^m$. Therefore
$
\mathbb P(E_2)
\ge \mathbb P(E_2\cap E_1)
\ge \varepsilon^m\mathbb P(E_1),
$
and consequently
$
\mathbb P(E_1)\le \varepsilon^{-m}\mathbb P(E_2).
$

We next bound $\mathbb P(E_2)$ by conditioning on the pooled sample and counting which $m$-point subsets can play the role of $S_2$. Write $[N]=\{1,\ldots,N\}$, and for a
sample $S=((x_i,y_i))_{i=1}^N$ and $I\subseteq[N]$, write $S|I$ for the
subsample indexed by $I$. By exchangeability of the pooled sample
$S\sim P^N$,
\[
\mathbb P(E_2)
=
\mathbb E_{S\sim P^N}\left[
\frac{1}{\binom{N}{m}}
\sum_{\substack{I\subseteq [N]\\ |I|=m}}
\ind\{I\in\mathcal A_S\}
\right],
\]
where
$
\mathcal A_S=
\left\{
I\subseteq[N], |I|=m:
\exists g\in\conv (\cH)
\text{ such that }
 \loss_{S|I^c}^{\gamma}(g)=0
\text{ and }
 \loss_{S|I}^{0}(g)=1
\right\}.
$

It remains to bound the number of admissible ghost-index sets. We do so through their VC-dimension: we claim that $\mathcal A_S$ has VC-dimension at most $d_\gamma$. Let
$J\subseteq[N]$ be shattered by $\mathcal A_S$, and write $k=|J|$. For each
sign vector $(\sigma_i)_{i\in J}\in\{-1,+1\}^J$, there is a set
$I_\sigma\in\mathcal A_S$ such that
$
I_\sigma\cap J=\{i\in J:\sigma_i=-1\}.
$
For each $\sigma$, choose a function $g_\sigma\in\conv(\cH)$ witnessing that $I_\sigma\in\mathcal A_S$. Thus
\[
y_i g_\sigma(x_i)>\gamma \quad\text{when }\sigma_i=+1,
\qquad
y_i g_\sigma(x_i)\le 0 \quad\text{when }\sigma_i=-1.
\]
For this witness, every term $\sigma_i(y_i g_\sigma(x_i)-\gamma/2)$ with $i\in J$ is at least $\gamma/2$. Consequently, for every sign vector,
\[
\begin{aligned}
\frac{\gamma}{2}
&\le
\sup_{g\in\conv(\cH)}
\frac{1}{k}\sum_{i\in J}\sigma_i
\left(y_i g(x_i)-\frac{\gamma}{2}\right)\\
&=
\sup_{g\in\conv(\cH)}
\frac{1}{k}\sum_{i\in J}\sigma_i y_i g(x_i)
-\frac{\gamma}{2k}\sum_{i\in J}\sigma_i\\
&=
\sup_{\substack{g\in\conv(\cH)\\
g=\sum_{h\in\cH}\alpha_h h}}
\sum_{h\in\cH}\alpha_h
\frac{1}{k}\sum_{i\in J}\sigma_i y_i h(x_i)
-\frac{\gamma}{2k}\sum_{i\in J}\sigma_i\\
&=
\sup_{h\in\cH}
\frac{1}{k}\sum_{i\in J}\sigma_i y_i h(x_i)
-\frac{\gamma}{2k}\sum_{i\in J}\sigma_i .
\end{aligned}
\]
Taking expectations over independent Rademacher signs and using $\mathbb E_\sigma[\sigma_i]=0$ gives
\[
\frac{\gamma}{2}
\le
\mathbb E_\sigma
\sup_{h\in\cH}
\frac{1}{k}\sum_{i\in J}\sigma_i y_i h(x_i)
\le
\sqrt{\frac{c'd}{k}},
\]
because the labels $y_i$ are fixed, $(\sigma_i y_i)_{i\in J}$ is again a Rademacher sequence; the last inequality is therefore the standard Rademacher bound for VC-classes (see \cite[Theorem 5.6]{AFoL_Lecture05}). Taking $c\ge 4c'$ yields
$k\le c d/\gamma^2=d_\gamma$, proving the claim.

Since $\varepsilon<1$, $\varepsilon\geq 10d_\gamma/n$ and $m=\lfloor \varepsilon n/2\rfloor$, we have $d_\gamma<n/10<n+m=N$. Sauer's lemma therefore gives, for every fixed sample $S$,
$
|\mathcal A_S|
\le
\sum_{j=0}^{\lfloor d_\gamma\rfloor}\binom{N}{j}
\le
\left(\frac{eN}{d_\gamma}\right)^{d_\gamma}.
$ Thus, using $\binom{N}{m}\geq(N/m)^m$ gives
\[
\mathbb P(E_2)
\le
\frac{1}{\binom{N}{m}}
\left(\frac{eN}{d_\gamma}\right)^{d_\gamma}
\le
\left(\frac{m}{N}\right)^m
\left(\frac{eN}{d_\gamma}\right)^{d_\gamma}.
\]
Combining this with the lower bound on $\mathbb P(E_2)$ and using $m=\lfloor\varepsilon n/2\rfloor\leq\varepsilon N/2$ gives
\[
\mathbb P(E_1)
\le
\left(\frac{eN}{d_\gamma}\right)^{d_\gamma}
\left(\frac{m}{\varepsilon N}\right)^m
\le
\left(\frac{eN}{d_\gamma}\right)^{d_\gamma}2^{-m}.
\]
Since $\varepsilon n\ge 10\ln(e/\delta)\ge 10$, we have
$m=\lfloor\varepsilon n/2\rfloor\ge \varepsilon n/4$. Also
$\varepsilon< 1$ implies $N=n+m\le 2n$. Hence
\[
\begin{aligned}
\mathbb P(E_1)
&\le
\exp\left(
d_\gamma\ln\left(\frac{eN}{d_\gamma}\right)
-m\ln 2
\right)\\
&\le
\exp\left(
d_\gamma\ln\left(\frac{6n}{d_\gamma}\right)
-\frac{\varepsilon n}{4}\ln 2
\right)\\
&\le \delta,
\end{aligned}
\]
by the choice of $\varepsilon$ and the definition of $\Ln$. This proves the
theorem, after absorbing the numerical constant into the universal constant
$C$.
\end{proof}

\section*{Acknowledgments}
While doing this work, Mikael Møller Høgsgaard was supported by a Carlsberg Internationalisation Fellowships.
\paragraph{LLM Usage and Acknowledgment.}
The proof idea for \Cref{thm:voting-margin-bound} arose from the authors' interactions with ChatGPT 5.5. Although the authors take full responsibility for the content of this work and the correctness of the proof, they acknowledge the LLM's significant contribution to the work.


\begingroup
\printbibliography

@inproceedings{HogsgaardLR23,
  author       = {Mikael M{\o}ller H{\o}gsgaard and
                  Kasper Green Larsen and
                  Martin Ritzert},
  editor       = {Andreas Krause and
                  Emma Brunskill and
                  Kyunghyun Cho and
                  Barbara Engelhardt and
                  Sivan Sabato and
                  Jonathan Scarlett},
  title        = {AdaBoost is not an Optimal Weak to Strong Learner},
  booktitle    = {International Conference on Machine Learning, {ICML} 2023, 23-29 July
                  2023, Honolulu, Hawaii, {USA}},
  series       = {Proceedings of Machine Learning Research},
  volume       = {202},
  pages        = {13118--13140},
  publisher    = {{PMLR}},
  year         = {2023},
  url          = {https://proceedings.mlr.press/v202/hogsgaard23a.html},
  bibsource    = {dblp computer science bibliography, https://dblp.org}
}

@inproceedings{FreundS95,
  author       = {Yoav Freund and
                  Robert E. Schapire},
  editor       = {Paul M. B. Vit{\'{a}}nyi},
  title        = {A decision-theoretic generalization of on-line learning and an application
                  to boosting},
  booktitle    = {Computational Learning Theory, Second European Conference, EuroCOLT
                  '95, Barcelona, Spain, March 13-15, 1995, Proceedings},
  series       = {Lecture Notes in Computer Science},
  volume       = {904},
  pages        = {23--37},
  publisher    = {Springer},
  year         = {1995},
  url          = {https://doi.org/10.1007/3-540-59119-2\_166},
  doi          = {10.1007/3-540-59119-2\_166},
  bibsource    = {dblp computer science bibliography, https://dblp.org}
}

@article{KearnsV94,
  author       = {Michael J. Kearns and
                  Leslie G. Valiant},
  title        = {Cryptographic Limitations on Learning Boolean Formulae and Finite
                  Automata},
  journal      = {J. {ACM}},
  volume       = {41},
  number       = {1},
  pages        = {67--95},
  year         = {1994},
  url          = {https://doi.org/10.1145/174644.174647},
  doi          = {10.1145/174644.174647},
  bibsource    = {dblp computer science bibliography, https://dblp.org}
}

@inproceedings{Valiant84,
  author       = {Leslie G. Valiant},
  editor       = {Richard A. DeMillo},
  title        = {A Theory of the Learnable},
  booktitle    = {Proceedings of the 16th Annual {ACM} Symposium on Theory of Computing,
                  April 30 - May 2, 1984, Washington, DC, {USA}},
  pages        = {436--445},
  publisher    = {{ACM}},
  year         = {1984},
  url          = {https://doi.org/10.1145/800057.808710},
  doi          = {10.1145/800057.808710},
  bibsource    = {dblp computer science bibliography, https://dblp.org}
}

@article{BlumerEHW89,
  author       = {Anselm Blumer and
                  Andrzej Ehrenfeucht and
                  David Haussler and
                  Manfred K. Warmuth},
  title        = {Learnability and the Vapnik-Chervonenkis dimension},
  journal      = {J. {ACM}},
  volume       = {36},
  number       = {4},
  pages        = {929--965},
  year         = {1989},
  url          = {https://doi.org/10.1145/76359.76371},
  doi          = {10.1145/76359.76371},
  bibsource    = {dblp computer science bibliography, https://dblp.org}
}

@article{VapnikC71,
 Author = {Vapnik, Vladimir and Chervonenkis, Alexey},
 Title = {On the uniform convergence of relative frequencies of events to their probabilities},
 journal = {Theory of Probability and its Applications},
 Volume = {16},
 Pages = {264--280},
 Year = {1971},
 Language = {English},
 DOI = {10.1137/1116025},
 zbMATH = {3391742},
 Zbl = {0247.60005}
}

@inproceedings{Schapire89,
  author       = {Robert E. Schapire},
  title        = {The Strength of Weak Learnability (Extended Abstract)},
  booktitle    = {30th Annual Symposium on Foundations of Computer Science, Research
                  Triangle Park, North Carolina, USA, 30 October - 1 November 1989},
  pages        = {28--33},
  publisher    = {{IEEE} Computer Society},
  year         = {1989},
  url          = {https://doi.org/10.1109/SFCS.1989.63451},
  doi          = {10.1109/SFCS.1989.63451},
  bibsource    = {dblp computer science bibliography, https://dblp.org}
}

@article{ericksonP25,
  title={TabArena: A Living Benchmark for Machine Learning on Tabular Data},
  author={Erickson, Nick and Purucker, Lennart and others},
  journal={arXiv preprint arXiv:2506.16791},
  year={2025},
  url={https://huggingface.co/spaces/TabArena/leaderboard}
}

@inproceedings{KeMFYLM23,
 author = {Ke, Guolin and Meng, Qi and Finley, Thomas and Wang, Taifeng and Chen, Wei and Ma, Weidong and Ye, Qiwei and Liu, Tie-Yan},
 booktitle = {Advances in Neural Information Processing Systems},
 editor = {I. Guyon and U. Von Luxburg and S. Bengio and H. Wallach and R. Fergus and S. Vishwanathan and R. Garnett},
 pages = {},
 publisher = {Curran Associates, Inc.},
 title = {LightGBM: A Highly Efficient Gradient Boosting Decision Tree},
 url = {https://proceedings.neurips.cc/paper_files/paper/2017/file/6449f44a102fde848669bdd9eb6b76fa-Paper.pdf},
 volume = {30},
 year = {2017}
}

@inproceedings{ProkhorenkovaGVDG18,
author = {Prokhorenkova, Liudmila and Gusev, Gleb and Vorobev, Aleksandr and Dorogush, Anna Veronika and Gulin, Andrey},
title = {CatBoost: unbiased boosting with categorical features},
year = {2018},
publisher = {Curran Associates Inc.},
address = {Red Hook, NY, USA},
booktitle = {Proceedings of the 32nd International Conference on Neural Information Processing Systems},
pages = {6639–6649},
numpages = {11},
location = {Montr\'{e}al, Canada},
series = {NIPS'18}
}

@inproceedings{ChenG16,
  author       = {Tianqi Chen and
                  Carlos Guestrin},
  editor       = {Balaji Krishnapuram and
                  Mohak Shah and
                  Alexander J. Smola and
                  Charu C. Aggarwal and
                  Dou Shen and
                  Rajeev Rastogi},
  title        = {XGBoost: {A} Scalable Tree Boosting System},
  booktitle    = {Proceedings of the 22nd {ACM} {SIGKDD} International Conference on
                  Knowledge Discovery and Data Mining, San Francisco, CA, USA, August
                  13-17, 2016},
  pages        = {785--794},
  publisher    = {{ACM}},
  year         = {2016},
  url          = {https://doi.org/10.1145/2939672.2939785},
  doi          = {10.1145/2939672.2939785},
  bibsource    = {dblp computer science bibliography, https://dblp.org}
}

@inproceedings{HogsgaardG25,
  title = 	 {Improved Margin Generalization Bounds for Voting Classifiers},
  author =       {H\o{}gsgaard M\o{}ller, Mikael and Green Larsen, Kasper},
  booktitle = 	 {Proceedings of Thirty Eighth Conference on Learning Theory},
  pages = 	 {2822--2855},
  year = 	 {2025},
  editor = 	 {Haghtalab, Nika and Moitra, Ankur},
  volume = 	 {291},
  series = 	 {Proceedings of Machine Learning Research},
  publisher =    {PMLR},
  url = 	 {https://proceedings.mlr.press/v291/hogsgaard-moller25a.html}
}

@article{SrebroST10,
  author       = {Nathan Srebro and
                  Karthik Sridharan and
                  Ambuj Tewari},
  title        = {Smoothness, Low-Noise and Fast Rates},
  journal      = {CoRR},
  volume       = {abs/1009.3896},
  year         = {2010},
  url          = {http://arxiv.org/abs/1009.3896},
  eprinttype   = {arXiv},
  eprint       = {1009.3896},
  bibsource    = {dblp computer science bibliography, https://dblp.org}
}

@misc{LarsenS25,
      title={Tight Margin-Based Generalization Bounds for Voting Classifiers over Finite Hypothesis Sets},
      author={Kasper Green Larsen and Natascha Schalburg},
      year={2025},
      eprint={2511.20407},
      archivePrefix={arXiv},
      primaryClass={cs.LG},
      url={https://arxiv.org/abs/2511.20407},
}

@book{SchapireF12,
    author = {Schapire, Robert E. and Freund, Yoav},
    title = {Boosting: Foundations and Algorithms},
    publisher = {The MIT Press},
    year = {2012},
    month = {05},
    isbn = {9780262301183},
    doi = {10.7551/mitpress/8291.001.0001},
    url = {https://doi.org/10.7551/mitpress/8291.001.0001},
    eprint = {https://direct.mit.edu/book-pdf/2280056/book_9780262301183.pdf},
}

@article{BartlettFLS98,
author = {Peter Bartlett and Yoav Freund and Wee Sun Lee and Robert E. Schapire},
title = {{Boosting the margin: a new explanation for the effectiveness of voting methods}},
volume = {26},
journal = {The Annals of Statistics},
number = {5},
publisher = {Institute of Mathematical Statistics},
pages = {1651 -- 1686},
year = {1998},
doi = {10.1214/aos/1024691352},
URL = {https://doi.org/10.1214/aos/1024691352}
}

@article{Breiman99,
    author = {Breiman, Leo},
    title = {Prediction Games and Arcing Algorithms},
    journal = {Neural Computation},
    volume = {11},
    number = {7},
    pages = {1493-1517},
    year = {1999},
    month = {10},
    issn = {0899-7667},
    doi = {10.1162/089976699300016106},
    url = {https://doi.org/10.1162/089976699300016106},
    eprint = {https://direct.mit.edu/neco/article-pdf/11/7/1493/814214/089976699300016106.pdf},
}

@article{KoltchinskiiP02,
author = {V. Koltchinskii and D. Panchenko},
title = {{Empirical Margin Distributions and Bounding the Generalization Error of Combined Classifiers}},
volume = {30},
journal = {The Annals of Statistics},
number = {1},
publisher = {Institute of Mathematical Statistics},
pages = {1 -- 50},
year = {2002},
doi = {10.1214/aos/1015362183},
URL = {https://doi.org/10.1214/aos/1015362183}
}

@article{GaoZH13,
title = {On the doubt about margin explanation of boosting},
journal = {Artificial Intelligence},
volume = {203},
pages = {1-18},
year = {2013},
issn = {0004-3702},
doi = {https://doi.org/10.1016/j.artint.2013.07.002},
url = {https://www.sciencedirect.com/science/article/pii/S0004370213000684},
author = {Wei Gao and Zhi-Hua Zhou}
}

@inbook{GronlundKLMN19,
author = {Gr\o{}nlund, Allan and Kamma, Lior and Larsen, Kasper Green and Mathiasen, Alexander and Nelson, Jelani},
title = {Margin-based generalization lower bounds for boosted classifiers},
year = {2019},
publisher = {Curran Associates Inc.},
address = {Red Hook, NY, USA},
booktitle = {Proceedings of the 33rd International Conference on Neural Information Processing Systems},
articleno = {1071},
numpages = {10}
}

@inproceedings{GronlundKL20,
author = {Gr\o{}nlund, Allan and Kamma, Lior and Larsen, Kasper Green},
title = {Margins are insufficient for explaining gradient boosting},
year = {2020},
isbn = {9781713829546},
publisher = {Curran Associates Inc.},
address = {Red Hook, NY, USA},
booktitle = {Proceedings of the 34th International Conference on Neural Information Processing Systems},
articleno = {160},
numpages = {11},
location = {Vancouver, BC, Canada},
series = {NIPS '20}
}

@inproceedings{Ben-DavidLM01,
  author    = {Shai Ben{-}David and
               Philip M. Long and
               Yishay Mansour},
  editor    = {David P. Helmbold and
               Robert C. Williamson},
  title     = {Agnostic Boosting},
  booktitle = {Computational Learning Theory, 14th Annual Conference on Computational
               Learning Theory, {COLT} 2001 and 5th European Conference on Computational
               Learning Theory, EuroCOLT 2001, Amsterdam, The Netherlands, July 16-19,
               2001, Proceedings},
  series    = {Lecture Notes in Computer Science},
  volume    = {2111},
  pages     = {507--516},
  publisher = {Springer},
  year      = {2001},
  url       = {https://doi.org/10.1007/3-540-44581-1\_33},
  doi       = {10.1007/3-540-44581-1\_33},
  bibsource = {dblp computer science bibliography, https://dblp.org}
}

@inproceedings{KalaiK09,
  author    = {Adam Kalai and
               Varun Kanade},
  editor    = {Yoshua Bengio and
               Dale Schuurmans and
               John D. Lafferty and
               Christopher K. I. Williams and
               Aron Culotta},
  title     = {Potential-Based Agnostic Boosting},
  booktitle = {Advances in Neural Information Processing Systems 22: 23rd Annual
               Conference on Neural Information Processing Systems 2009. Proceedings
               of a meeting held 7-10 December 2009, Vancouver, British Columbia,
               Canada},
  pages     = {880--888},
  publisher = {Curran Associates, Inc.},
  year      = {2009},
  url       = {https://proceedings.neurips.cc/paper/2009/hash/13f9896df61279c928f19721878fac41-Abstract.html},
  bibsource = {dblp computer science bibliography, https://dblp.org}
}

@inproceedings{Feldman10,
  author    = {Vitaly Feldman},
  editor    = {Andrew Chi{-}Chih Yao},
  title     = {Distribution-Specific Agnostic Boosting},
  booktitle = {Innovations in Computer Science - {ICS} 2010, Tsinghua University,
               Beijing, China, January 5-7, 2010. Proceedings},
  pages     = {241--250},
  publisher = {Tsinghua University Press},
  year      = {2010},
  url       = {http://conference.iiis.tsinghua.edu.cn/ICS2010/content/papers/20.html},
  bibsource = {dblp computer science bibliography, https://dblp.org}
}

@inproceedings{BrukhimCHM20,
  author    = {Nataly Brukhim and
               Xinyi Chen and
               Elad Hazan and
               Shay Moran},
  editor    = {Hugo Larochelle and
               Marc'Aurelio Ranzato and
               Raia Hadsell and
               Maria{-}Florina Balcan and
               Hsuan{-}Tien Lin},
  title     = {Online Agnostic Boosting via Regret Minimization},
  booktitle = {Advances in Neural Information Processing Systems 33: Annual Conference
               on Neural Information Processing Systems 2020, NeurIPS 2020, December
               6-12, 2020, virtual},
  year      = {2020},
  url       = {https://proceedings.neurips.cc/paper/2020/hash/07168af6cb0ef9f78dae15739dd73255-Abstract.html},
  bibsource = {dblp computer science bibliography, https://dblp.org}
}

@inproceedings{GhaiS24,
  author    = {Udaya Ghai and
               Karan Singh},
  editor    = {Amir Globersons and
               Lester Mackey and
               Danielle Belgrave and
               Angela Fan and
               Ulrich Paquet and
               Jakub M. Tomczak and
               Cheng Zhang},
  title     = {Sample-Efficient Agnostic Boosting},
  booktitle = {Advances in Neural Information Processing Systems 38: Annual Conference
               on Neural Information Processing Systems 2024, NeurIPS 2024, Vancouver,
               BC, Canada, December 10 - 15, 2024},
  year      = {2024},
  url       = {http://papers.nips.cc/paper\_files/paper/2024/hash/b63a24a1832bd14fa945c71f535c0095-Abstract-Conference.html},
  bibsource = {dblp computer science bibliography, https://dblp.org}
}

@article{GhaiS25,
  author     = {Udaya Ghai and
                Karan Singh},
  title      = {Sample-Optimal Agnostic Boosting with Unlabeled Data},
  journal    = {CoRR},
  volume     = {abs/2503.04706},
  year       = {2025},
  url        = {https://doi.org/10.48550/arXiv.2503.04706},
  doi        = {10.48550/ARXIV.2503.04706},
  eprinttype = {arXiv},
  eprint     = {2503.04706},
  bibsource  = {dblp computer science bibliography, https://dblp.org}
}

@article{DaCunhaHPS25,
  author     = {Arthur {da Cunha} and
                Mikael M{\o}ller H{\o}gsgaard and
                Andrea Paudice and
                Yuxin Sun},
  title      = {Revisiting Agnostic Boosting},
  journal    = {CoRR},
  volume     = {abs/2503.09384},
  year       = {2025},
  url        = {https://doi.org/10.48550/arXiv.2503.09384},
  doi        = {10.48550/ARXIV.2503.09384},
  eprinttype = {arXiv},
  eprint     = {2503.09384},
  bibsource  = {dblp computer science bibliography, https://dblp.org}
}

@misc{DaCunhaHP26,
      title={Sample-Near-Optimal Agnostic Boosting with Improved Running Time},
      author={Arthur da Cunha and Mikael Møller Høgsgaard and Andrea Paudice},
      year={2026},
      eprint={2601.11265},
      archivePrefix={arXiv},
      primaryClass={cs.LG},
      url={https://arxiv.org/abs/2601.11265},
}

@InProceedings{karbasiL24,
  title = 	 {The Impossibility of Parallelizing Boosting},
  author =       {Karbasi, Amin and Green Larsen, Kasper},
  booktitle = 	 {Proceedings of The 35th International Conference on Algorithmic Learning Theory},
  pages = 	 {635--653},
  year = 	 {2024},
  editor = 	 {Vernade, Claire and Hsu, Daniel},
  volume = 	 {237},
  series = 	 {Proceedings of Machine Learning Research},
  month = 	 {25--28 Feb},
  publisher =    {PMLR},
  url = 	 {https://proceedings.mlr.press/v237/karbasi24a.html}
}

@inproceedings{daCunhaHL24,
 author = {da Cunha, Arthur and M\o ller H\o gsgaard, Mikael and Larsen, Kasper Green},
 booktitle = {Advances in Neural Information Processing Systems},
 doi = {10.52202/079017-0464},
 editor = {A. Globerson and L. Mackey and D. Belgrave and A. Fan and U. Paquet and J. Tomczak and C. Zhang},
 pages = {14540--14569},
 publisher = {Curran Associates, Inc.},
 title = {Optimal Parallelization of Boosting},
 url = {https://proceedings.neurips.cc/paper_files/paper/2024/file/1a675d804f50509b8e21d0d3ca709d03-Paper-Conference.pdf},
 volume = {37},
 year = {2024}
}

@article{LongS13,
  author  = {Philip M. Long and Rocco A. Servedio},
  title   = {Algorithms and Hardness Results for Parallel Large Margin Learning},
  journal = {Journal of Machine Learning Research},
  year    = {2013},
  volume  = {14},
  number  = {95},
  pages   = {3105--3128},
  url     = {http://jmlr.org/papers/v14/long13a.html}
}

@inbook{XinHJ24,
author = {Xin Lyu and Hongxun Wu and Junzhao Yang},
title = {The Cost of Parallelizing Boosting},
booktitle = {Proceedings of the 2024 Annual ACM-SIAM Symposium on Discrete Algorithms (SODA)},
chapter = {},
pages = {3140-3155},
doi = {10.1137/1.9781611977912.112},
URL = {https://epubs.siam.org/doi/abs/10.1137/1.9781611977912.112},
eprint = {https://epubs.siam.org/doi/pdf/10.1137/1.9781611977912.112}
}

@article{MukherjeeS13,
  author  = {Indraneel Mukherjee and Robert E. Schapire},
  title   = {A Theory of Multiclass Boosting},
  journal = {Journal of Machine Learning Research},
  year    = {2013},
  volume  = {14},
  number  = {14},
  pages   = {437--497},
  url     = {http://jmlr.org/papers/v14/mukherjee13a.html}
}

@inproceedings{BrukhimDM23,
 author = {Brukhim, Nataly and Daniely, Amit and Mansour, Yishay and Moran, Shay},
 booktitle = {Advances in Neural Information Processing Systems},
 editor = {A. Oh and T. Naumann and A. Globerson and K. Saenko and M. Hardt and S. Levine},
 pages = {1403--1425},
 publisher = {Curran Associates, Inc.},
 title = {Multiclass Boosting: Simple and Intuitive Weak Learning Criteria},
 url = {https://proceedings.neurips.cc/paper_files/paper/2023/file/050f8591be3874b52fdac4e1060eeb29-Paper-Conference.pdf},
 volume = {36},
 year = {2023}
}

@inproceedings{SchapireS98,
author = {Schapire, Robert E. and Singer, Yoram},
title = {Improved boosting algorithms using confidence-rated predictions},
year = {1998},
isbn = {1581130570},
publisher = {Association for Computing Machinery},
address = {New York, NY, USA},
url = {https://doi.org/10.1145/279943.279960},
doi = {10.1145/279943.279960},
booktitle = {Proceedings of the Eleventh Annual Conference on Computational Learning Theory},
pages = {80–91},
numpages = {12},
location = {Madison, Wisconsin, USA},
series = {COLT' 98}
}

@InProceedings{bressan25a,
  title = 	 {Of Dice and Games: A Theory of Generalized Boosting},
  author =       {Bressan, Marco and Brukhim, Nataly and Cesa-Bianchi, Nicol{\`o} and Esposito, Emmanuel and Mansour, Yishay and Moran, Shay and Thiessen, Maximilian},
  booktitle = 	 {Proceedings of Thirty Eighth Conference on Learning Theory},
  pages = 	 {596--640},
  year = 	 {2025},
  editor = 	 {Haghtalab, Nika and Moitra, Ankur},
  volume = 	 {291},
  series = 	 {Proceedings of Machine Learning Research},
  month = 	 {30 Jun--04 Jul},
  publisher =    {PMLR},
  url = 	 {https://proceedings.mlr.press/v291/bressan25a.html}
}

@InProceedings{Servedio01,
author="Servedio, Rocco A.",
editor="Helmbold, David
and Williamson, Bob",
title="Smooth Boosting and Learning with Malicious Noise",
booktitle="Computational Learning Theory",
year="2001",
publisher="Springer Berlin Heidelberg",
address="Berlin, Heidelberg",
pages="473--489",
isbn="978-3-540-44581-4"
}

@InProceedings{Gavinsky02,
author="Gavinsky, Dmitry",
editor="Cesa-Bianchi, Nicol{\`o}
and Numao, Masayuki
and Reischuk, R{\"u}diger",
title="Optimally-Smooth Adaptive Boosting and Application to Agnostic Learning",
booktitle="Algorithmic Learning Theory",
year="2002",
publisher="Springer Berlin Heidelberg",
address="Berlin, Heidelberg",
pages="98--112",
isbn="978-3-540-36169-5"
}

@inbook{Barak09,
author = {Boaz Barak and Moritz Hardt and Satyen Kale},
title = {The Uniform Hardcore Lemma via Approximate Bregman Projections},
booktitle = {Proceedings of the 2009 Annual ACM-SIAM Symposium on Discrete Algorithms (SODA)},
chapter = {},
pages = {1193-1200},
doi = {10.1137/1.9781611973068.129},
URL = {https://epubs.siam.org/doi/abs/10.1137/1.9781611973068.129},
eprint = {https://epubs.siam.org/doi/pdf/10.1137/1.9781611973068.129}
}

@InProceedings{BunCS20,
  title = 	 {Efficient, Noise-Tolerant, and Private Learning via Boosting},
  author =       {Bun, Mark and Carmosino, Marco Leandro and Sorrell, Jessica},
  booktitle = 	 {Proceedings of Thirty Third Conference on Learning Theory},
  pages = 	 {1031--1077},
  year = 	 {2020},
  editor = 	 {Abernethy, Jacob and Agarwal, Shivani},
  volume = 	 {125},
  series = 	 {Proceedings of Machine Learning Research},
  month = 	 {09--12 Jul},
  publisher =    {PMLR},
  url = 	 {https://proceedings.mlr.press/v125/bun20a.html}
}

@article{Drucker97,
author = {Drucker, Harris},
year = {1997},
month = {08},
pages = {},
title = {Improving Regressors Using Boosting Techniques},
journal = {Proceedings of the 14th International Conference on Machine Learning}
}

@article{DuffyH02,
author = {Duffy, Nigel and Helmbold, David},
title = {Boosting Methods for Regression},
year = {2002},
issue_date = {May-June 2002},
publisher = {Kluwer Academic Publishers},
address = {USA},
volume = {47},
number = {2–3},
issn = {0885-6125},
url = {https://doi.org/10.1023/A:1013685603443},
doi = {10.1023/A:1013685603443},
journal = {Mach. Learn.},
month = may,
pages = {153–200},
numpages = {48}
}

@article{Friedman01,
author = {Jerome H. Friedman},
title = {{Greedy function approximation: A gradient boosting machine.}},
volume = {29},
journal = {The Annals of Statistics},
number = {5},
publisher = {Institute of Mathematical Statistics},
pages = {1189 -- 1232},
year = {2001},
doi = {10.1214/aos/1013203451},
URL = {https://doi.org/10.1214/aos/1013203451}
}

@article{BühlmannB03,
author = {Peter Bühlmann and Bin Yu},
title = {Boosting With the L2 Loss},
journal = {Journal of the American Statistical Association},
volume = {98},
number = {462},
pages = {324--339},
year = {2003},
publisher = {Taylor \& Francis},
doi = {10.1198/016214503000125},


URL = {

        https://doi.org/10.1198/016214503000125



},
eprint = {

        https://doi.org/10.1198/016214503000125



}

}

@inproceedings{ImpagliazzoLPS22,
author = {Impagliazzo, Russell and Lei, Rex and Pitassi, Toniann and Sorrell, Jessica},
title = {Reproducibility in learning},
year = {2022},
isbn = {9781450392648},
publisher = {Association for Computing Machinery},
address = {New York, NY, USA},
url = {https://doi.org/10.1145/3519935.3519973},
doi = {10.1145/3519935.3519973},
booktitle = {Proceedings of the 54th Annual ACM SIGACT Symposium on Theory of Computing},
pages = {818–831},
numpages = {14},
location = {Rome, Italy},
series = {STOC 2022}
}

@InProceedings{larsenMS26,
  title = 	 {Improved Replicable Boosting with Majority-of-Majorities},
  author =       {Larsen, Kasper Green and Mathiasen, Markus Engelund and Svendsen, Clement},
  booktitle = 	 {Proceedings of The 37th International Conference on Algorithmic Learning Theory},
  pages = 	 {1--18},
  year = 	 {2026},
  editor = 	 {Telgarsky, Matus and Ullman, Jonathan},
  volume = 	 {313},
  series = 	 {Proceedings of Machine Learning Research},
  month = 	 {23--26 Feb},
  publisher =    {PMLR},
  url = 	 {https://proceedings.mlr.press/v313/larsen26b.html}
}

@misc{AFoL_Lecture05,
  author       = {Patrick Rebeschini},
  title        = {{Lecture Notes in Algorithmic Foundations of Learning: Covering Numbers Bounds for Rademacher Complexity. Chaining}},
  month        = dec,
  year         = {2021},
  url          = {https://www.stats.ox.ac.uk/~rebeschi/teaching/AFoL/22/material/lecture05.pdf},
  note         = {Lecture 5 (Chaining), Department of Statistics, University of Oxford. Version of December 8, 2021},
}

@InProceedings{cunhaLR25,
  title = 	 {Boosting, Voting Classifiers and Randomized Sample Compression Schemes},
  author =       {da Cunha, Arthur and Larsen, Kasper Green and Ritzert, Martin},
  booktitle = 	 {Proceedings of The 36th International Conference on Algorithmic Learning Theory},
  pages = 	 {390--404},
  year = 	 {2025},
  editor = 	 {Kamath, Gautam and Loh, Po-Ling},
  volume = 	 {272},
  series = 	 {Proceedings of Machine Learning Research},
  month = 	 {24--27 Feb},
  publisher =    {PMLR},
  url = 	 {https://proceedings.mlr.press/v272/cunha25a.html}
}
\endgroup

\appendix\section[Calculations for the AdaBoost margin bound]{Calculations Showing That $T\geq \ln(n)/\gamma^2$ Implies $\loss_S^{\gamma/2}(g)=0$}\label[appendix]{sec:appendix}
    To prove the final assertion of \Cref{lem:adaboost-margin}, it suffices to show that the logarithm of the factor raised to the $T$th power is less than $-\gamma^2$. We begin with the following identity:
    \begin{align*}
     \ln{(\sqrt{(1-2\gamma)^{1-\gamma/2}(1+2\gamma)^{1+\gamma/2}})}
     &=\frac{1}{2}\left(\left(1-\frac{\gamma}{2}\right)\ln{(1-2\gamma)}+\left(1+\frac{\gamma}{2}\right)\ln{(1+2\gamma)}\right)
     \\
     &=\frac{1}{2}\left(\ln{\left((1-2\gamma)(1+2\gamma)\right)}+\frac{\gamma}{2}\ln{\left(\frac{1+2\gamma}{1-2\gamma}\right)}\right)
     \\
     &=\frac{1}{2}\left(\ln{(1-4\gamma^{2})}+\frac{\gamma}{2}\ln{\left(\frac{1+2\gamma}{1-2\gamma}\right)}\right).
    \end{align*}
    For $|x|<1$, we use the expansion
    $
     \ln{(1-x)}=-\sum_{k=1}^{\infty} \frac{x^{k}}{k}.
    $
    Furthermore, we have
    \begin{align*}
        \ln{\left(\frac{1+x}{1-x}\right)}
        =
        \ln{(1+x)}-\ln{(1-x)}
        =-\sum_{k=1}^{\infty} \frac{(-x)^{k}}{k}-\left(-\sum_{k=1}^{\infty} \frac{x^{k}}{k}\right)
        =2\sum_{k=1}^{\infty} \frac{x^{2k-1}}{2k-1}.
    \end{align*}
    Consequently,
    $
     x\ln{\left(\frac{1+x}{1-x}\right)}=2\sum_{k=1}^{\infty} \frac{x^{2k}}{2k-1}.
   $
    Because $0<\gamma<1/2$ by the assumptions of \Cref{lem:adaboost-margin}, substituting $x=4\gamma^2$ into the logarithmic expansion and $x=2\gamma$ into the log-ratio expansion is admissible and gives
    \begin{align*}
      \ln{(\sqrt{(1-2\gamma)^{1-\gamma/2}(1+2\gamma)^{1+\gamma/2}})}
     &=\frac{1}{2}\left(\ln{(1-4\gamma^{2})}+\frac{\gamma}{2}\ln{\left(\frac{1+2\gamma}{1-2\gamma}\right)}\right)
     \\
     &=\frac{1}{2}\left(-\sum_{k=1}^{\infty} \frac{(4\gamma^{2})^{k}}{k}+\frac{1}{2}\sum_{k=1}^{\infty} \frac{(2\gamma)^{2k}}{2k-1}\right)
     \\
     &=\frac{1}{2}\left(-\sum_{k=1}^{\infty} \frac{(2\gamma)^{2k}}{k}+\frac{1}{2}\sum_{k=1}^{\infty} \frac{(2\gamma)^{2k}}{2k-1}\right)
     \\
     &=\sum_{k=1}^{\infty} \left(\frac{1}{4(2k-1)}-\frac{1}{2k}\right)(2\gamma)^{2k}
     <-\gamma^{2}.
    \end{align*}
    The last inequality follows because $\frac{1}{4(2k-1)}-\frac{1}{2k}<0$ for every $k\geq 2$, while $\frac{1}{4(2k-1)}-\frac{1}{2k}=-\frac14$ for $k=1$. Therefore, \Cref{lem:adaboost-margin} and the displayed logarithmic bound imply that, if $T\geq \ln(n)/\gamma^2$, then $\loss_S^{\gamma/2}(g)<e^{-T\gamma^2}\leq 1/n$. Since the empirical margin loss is a multiple of $1/n$, it follows that $\loss_S^{\gamma/2}(g)=0$. This proves the final assertion of \Cref{lem:adaboost-margin}.

\end{document}